\documentclass[letterpaper]{article} 
\usepackage{aaai25}  
\usepackage{times}  
\usepackage{helvet}  
\usepackage{courier}  
\usepackage[hyphens]{url}  
\usepackage{graphicx} 
\usepackage{natbib}  
\usepackage{caption} 
\usepackage{algorithm}
\usepackage{algorithmic}

\usepackage{xspace} 
\newcommand{\methodName}{DOMBA\xspace}
\usepackage{newfloat}
\usepackage{listings}
\DeclareCaptionStyle{ruled}{labelfont=normalfont,labelsep=colon,strut=off} 
\floatstyle{ruled}
\newfloat{listing}{tb}{lst}{}
\floatname{listing}{Listing}
\title{\methodName: Double Model Balancing for Access-Controlled\\Language Models via Minimum-Bounded Aggregation}
\author {
    Tom Segal,
    Asaf Shabtai,
    Yuval Elovici
}
\affiliations {
    Ben-Gurion University of the Negev\\
    tomsega@post.bgu.ac.il, shabtaia@bgu.ac.il, elovici@bgu.ac.il
}

\usepackage{amsmath, amsthm, amsfonts}

\newtheorem{definition}{Definition}
\newtheorem{lemma}{Lemma}
\newtheorem{theorem}{Theorem}
\newtheorem{corollary}{Corollary}

\begin{document}

\maketitle

\begin{abstract}
The utility of large language models (LLMs) depends heavily on the quality and quantity of their training data. 
Many organizations possess large data corpora that could be leveraged to train or fine-tune LLMs tailored to their specific needs. 
However, these datasets often come with access restrictions that are based on user privileges and enforced by access control mechanisms. Training LLMs on such datasets could result in exposure of sensitive information to unauthorized users.
A straightforward approach for preventing such exposure is to train a separate model for each access level. 
This, however, may result in low utility models due to the limited amount of training data per model compared to the amount in the entire organizational corpus. 
Another approach is to train a single LLM on all the data while limiting the exposure of unauthorized information. 
However, current exposure-limiting methods for LLMs are ineffective for access-controlled data, where sensitive information appears frequently across many training examples.
We propose \methodName -- double model balancing -- a simple approach for training and deploying LLMs that provides high utility and access-control functionality with security guarantees. 
\methodName aggregates the probability distributions of two models, each trained on documents with (potentially many) different access levels, using a ``min-bounded" average function (a function that is bounded by the smaller value, e.g., harmonic mean). 
A detailed mathematical analysis and extensive evaluation show that \methodName safeguards restricted information while offering utility comparable to non-secure models.
\end{abstract}

\begin{links}
\link{Code and datasets}{https://github.com/ppo1/DOMBA}
\end{links}

\section{Introduction}
Organizations can benefit greatly from training dedicated LLMs, such as coding assistants, email writers or question-answering models, on their data~\cite{tiwari2023information}. 
While the benefits can be substantial, such data often contains restricted information, and an access-control mechanism ensuring users can only access information according to their access rights is usually in place. 
However, LLMs inherently lack such access-control mechanisms, which can lead to the exposure of sensitive information to unauthorized users~\cite{274574,kandpal2024user,9152761}.
A basic approach for introducing access control to LLMs is to train a separate LLM for each access level~\cite{tiwari2023information}. 
However, as our experiments show, this approach can substantially reduce model utility, since the amount of data for each access level is limited. 
For example, training a model on emails from only one department in an organization may be insufficient for constructing effective organizational emails. To overcome this limitation, sufficient data (including restricted data) must be included in the model's training set.
This means that any secure and high utility method should limit the exposure of the training data to users of the model (according to their access rights).

\begin{figure*}[t]
\centering
\includegraphics[width=0.9\textwidth]{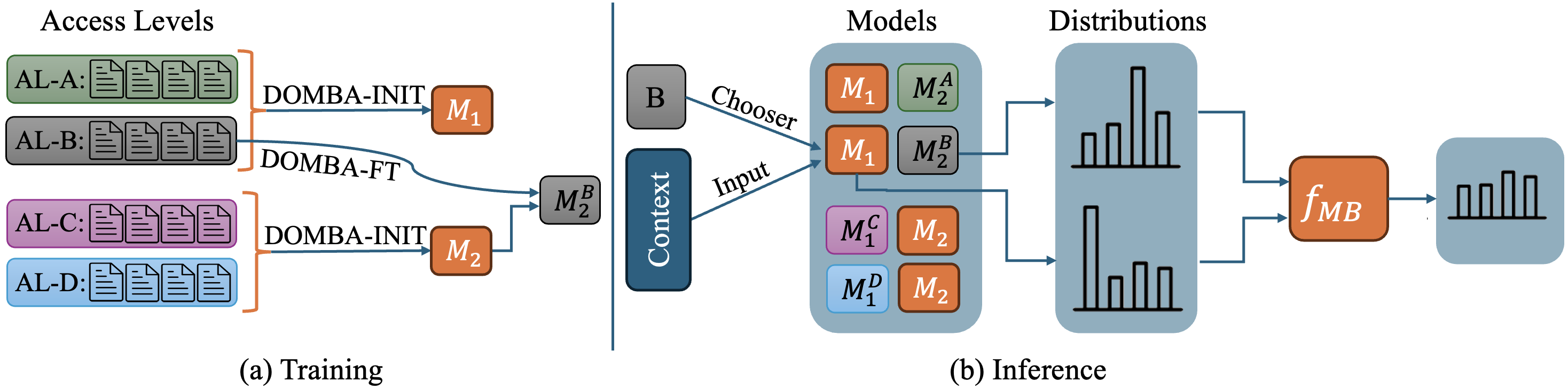} 
\caption{Two main phases of our proposed method. 
(a) Training: Documents are grouped and divided into two partitions according to their access levels (each access level is depicted in a different color). Training includes two steps:
DOMBA-INIT: A submodel is trained on each partition, resulting in $M_1$ and $M_2$. 
DOMBA-FT: To achieve a model dedicated to access level B (AL-B), $M_2$ is fine-tuned on AL-B documents, resulting in $M_2^B$ (note that AL-B documents are included in the training sets of both $M_1$ and $M_2^B$). 
(b) Inference: Given a context and access level, the corresponding submodels are selected, and their probability distributions are aggregated using a ``min-bounded" aggregation function $f_{MB}$.}
\label{fig1}
\end{figure*}

In this paper, we propose \methodName, a method for training and deploying LLMs that incorporates an access-control mechanism while maintaining high utility. 
An overview of \methodName is presented in Figure~\ref{fig1}.
To protect sensitive information, \methodName ``balances" two submodels (trained on two different data partitions, each including different access levels) during inference, using a min-bounded average function; intuitively, each submodel ``knows" different restricted information. 
During text generation, the min-bounded function makes it unlikely for information known to just one submodel to be generated. Assuming that restricted information is not shared between the two partitions, this ensures that no restricted information is likely to be generated. 
It is crucial to separate access levels into partitions in a way that groups access levels with shared sensitive information into the same partition. If we represent access levels as nodes in a graph and connect nodes with edges whenever they share sensitive information, such separation is achievable if the graph is disconnected (and, for practical purposes, does not contain an overly large connected component).

To analyze the privacy protection provided by \methodName, we formalize the notion of ``exposure of secrets" by developing a framework called ''token exposure." 
Token exposure is based on the change in probability of a token (relative to other tokens) between two language models. 
The greater the change in a token's probability (relative to other tokens) when using one model over the other, the greater the exposure of that token. 
Using our approach, the exposure of \methodName over both submodels is bounded by the best possible value (i.e., replacing \methodName by any other model would not improve the bound). 
This means that DOMBA limits the exposure of ``secrets" that appear in just one partition (since one of the submodels does not ``know" them).

Evaluating sensitive information exposure is challenging due to the fact that defining what constitutes sensitive information is a complex and context-dependent task, varying based on organizational policy and other factors~\cite{10.1145/3531146.3534642}. 
Measuring average-case sensitive information exposure is inadequate, since an adversary could devise a prompt that causes the model to substantially deviate from average behavior~\cite{wolf2024fundamental}. 
To address these challenges, we introduce three new security metrics for evaluating sensitive information exposure. 
The first metric assesses worst-case and ``extreme-case" exposure across data records. 
In the second metric, we evaluate the probabilities assigned to certain tokens that should not be exposed by a secure model. 
The third new metric involves a ``secret inference attack" (SIA) that is based on membership inference attack (MIA) techniques. 
In addition to evaluating \methodName on the three new metrics, we also evaluate it using the canary technique~\cite{236216}, in which specific phrases (canaries) are injected into the training data, and the model's inclination to generate them compared to similar phrases not included in the training data is measured. 

The contributions of this work are as follows:
\begin{enumerate}
    \item Providing the first practical and comprehensive solution for access control in LLMs, a solution that provides high utility and security by safely utilizing restricted data.
    \item Developing the mathematical framework of token exposure to analyse language models' exposure of sensitive information and establishing bounds for \methodName.
    \item Creating three novel empirical evaluation metrics for assessing sensitive information exposure and employing these metrics in our evaluation of \methodName.
\end{enumerate}

\section{Related Work}
Very few studies have addressed the use of LLMs in the access-control scenario: \citeauthor{tiwari2023information} \shortcite{tiwari2023information} proposed using mixture of experts (MoE) in conjunction with training a separate model for each access level in order to support users with multiple access rights. However, this approach relies solely on non-restricted documents, which substantially reduces its utility. \citeauthor{wutschitz2023rethinkingprivacymachinelearning} \shortcite{wutschitz2023rethinkingprivacymachinelearning} proposed using retrieval-augmented generation (RAG) with access rights, which prevents the retrieval of unauthorized documents. As illustrated by \citeauthor{tiwari2023information} \shortcite{tiwari2023information}, using RAG by itself, without training the model on the data, may be insufficient to achieve a substantial adjustment of LLM behaviors, such as altering writing styles or tones and incorporating new domain-specific knowledge and terminology. We consider RAG to be complementary to DOMBA, with the choice between them, or a combination of both, being highly dependent on the specific task and dataset. Determining the optimal approach for each case is beyond the scope of this work.

Several studies have explored the use of differential private training algorithms for deep learning~\cite{yu2022differentially,10.1145/2976749.2978318}. 
These algorithms are designed to provide privacy guarantees when each ``secret" appears in only one or a few data records. 
As demonstrated by \citeauthor{ginart2022submix} \shortcite{ginart2022submix}, scaling these algorithms to safeguard a partition of sensitive documents on the same topic (which is critical in the access control scenario) results in impractical utility.
Differential privacy (DP) federated learning~\cite{geyer2018differentially} is a specific application of DP in deep learning in which contributions from different clients are aggregated during training, and noise is added to maintain DP. 
While federated learning might seem promising for the access-control scenario (with each access level treated as a client), \citeauthor{9069945} \shortcite{9069945} demonstrated that performance drops substantially when using a small amount of clients (access levels). 
In contrast, \methodName inherently does not depend on the number of access levels. 

\citeauthor{PATE} \shortcite{PATE} proposed PATE, a differential private framework for training machine learning models for classification tasks. 
PATE is not suitable when there is a large amount of classes, and it requires an unlabeled non-private dataset. Given that all next-token prediction datasets are naturally labeled, and the number of tokens (which could be seen as classification classes) is very large for modern LLMs, PATE is not applicable in our case. 

\citeauthor{ginart2022submix} \shortcite{ginart2022submix} introduced SUBMIX, an inference-time partition-level differential private model ensemble, however their ensemble requires many models to provide meaningful privacy guarantees, resulting in both costly text generation and a degradation in utility. 

Beyond DP-based approaches, various techniques have been proposed to empirically mitigate the exposure of sensitive information. Such approaches include sanitization \cite{lison-etal-2021-anonymisation}, prompt engineering~\cite{chen2023language}, reinforcement learning from human feedback (RLHF)~\cite{NEURIPS2022_b1efde53} and using a privacy regularized loss function \cite{mireshghallah-etal-2021-privacy}.
However, these techniques lack theoretical guarantees, and some have been demonstrated to be vulnerable to attacks such as data extraction and ``jailbreaking" ~\cite{liu2024jailbreaking,chen2023language,nasr2023scalable,wolf2024fundamental}.




\section{Methodology}
In this section, we define the concept of exposing a secret and describe \methodName's training and aggregation processes. 
Using formal mathematical language, we establish a bound to \methodName's exposure of secrets. 
We also show that no other aggregation method could ever achieve a better bound.

\subsection{Training}
DOMBA-INIT: Let $d_1,...,d_k$ be the datasets corresponding to access levels $1,...,k$. 
We randomly assign each access level to one of two data partitions and train a submodel on each partition separately, denoted as $M_1$ and $M_2$. \\
DOMBA-FT: For each access level $AL$, let $M_1$ and $M_2$ be the resulting submodels of DOMBA-INIT. 
If $AL$ was assigned to $M_1$ during DOMBA-INIT, we fine-tune $M_2$ on $d_{AL}$. 
Otherwise, we fine-tune $M_1$ on $d_{AL}$. We then save the states of $M_1$ and $M_2$, which will be used during inference for users with access level $AL$.
If a user has multiple access rights, an MoE can be used as demonstrated by \citeauthor{tiwari2023information} \shortcite{tiwari2023information} (this scenario is not explored in this study).
We note that PEFT (parameter-efficient fine-tuning) methods such as LORA \cite{hu2021lora} can be used to efficiently train and store the different states of $M_1$ and $M_2$.

\subsection{Preliminaries}
Let $\Sigma$ be a set of tokens, and let $n=|\Sigma|$. 
We use $t$ to refer to a token and $c$ to refer to a context (i.e., a sequence of tokens preceding t). We use $M$, $M_1$, $M_2$ to refer to next-token prediction language models and denote the probability assigned by $M$ to token $t$ given context $c$ as $p_M(t|c)$ . 
We use $\sum$ (sum) without specification to indicate summation over all tokens (i.e., $\sum_{t \in \Sigma}$).
We note that the theorems in the subsections that follow commonly refer to arbitrary language models $M_1$ and $M_2$, but it may be helpful to think of $M_1$ and $M_2$ as the outputs of DOMBA-INIT or DOMBA-FT.

\subsection{Token Exposure}
We begin by defining exposing a secret in a formal sense. As highlighted by \citeauthor{10.1145/3531146.3534642} \shortcite{10.1145/3531146.3534642}, secrecy is relative -- something is deemed secret if it is known by some but unknown to others. 
Therefore, our concept of secrecy involves comparing probabilities assigned to a token by two models. One possible approach is to use the ratio of the probabilities assigned by the models to assess secrecy. However, this method has drawbacks. Consider the following probability distributions over the tokens $a,b,c,d$: $p_1=(0.7,0.1,0.1,0.1)$ and $p_2=(0.97,0.01,0.01,0.01)$. The probabilities' ratios ($p_1/p_2$) are $(0.72, 10, 10, 10)$. This implies that tokens $b,c$, and $d$ are ``secret" in $p_1$ compared to $p_2$. However, it seems more appropriate to consider $a$ as secret in $p_2$ compared to $p_1$, because $p_2$ assigns $a$ a probability that is 97 higher than all other tokens, whereas $p_1$ assigns it a probability that is only seven times higher. To address this, we compare the probability ratio of a token $t$ (between two models) to a ``typical probability ratio" (TPR).

\begin{definition} [Geometric mean]
Let $f: \Sigma \rightarrow \mathbb{R^{+}}$. The geometric mean of $f$ is $GM(f(t)) := \exp(\frac{1}{n}\sum \log(f(t)))$.
\end{definition}

\begin{definition} [TPR]
Let $c$ be a context, and let $M_1,M_2$ be language models. We define the ``TPR at $c$" of $M_1, M_2$ as $tpr_c(M_1, M_2)=GM(\frac{p_{M_1}(t|c)}{p_{M_2}(t|c)})$.
\end{definition}

\begin{definition} [Token exposure]
Let $c$ be a context, $t$ be a token, and $M_1, M_2$ be language models. We call $t$ ``$\alpha$-exposed by $M_1$ over $M_2$ at $c$" if $\frac{p_{M_1}(t|c)}{p_{M_2}(t|c) \cdot tpr_c(M_1, M_2)} = \alpha$. We also say that $t$ is ``$\leq$$ \alpha$-exposed" if $t$ is $\beta$-exposed for some $\beta \leq \alpha$.
\end{definition}

In other words, instead of directly dividing $p_{M_1}(t|c)$ by $p_{M_2}(t|c)$, we adjust $p_{M_2}(t|c)$ by multiplying it by the TPR. In the example discussed, the TPR is 5.18, which results in the following exposures of tokens $a,b,c,d$: $M_1$ over $M_2$: (0.14, 1.93, 1.93, 1.93) and $M_2$ over $M_1$: (7.19, 0.52, 0.52, 0.52). These values better reflect our intuition that $a$ is secret and not $b,c$, and / or $d$.

\subsection{Token Exposure Properties}
In this subsection, we explore certain properties of token exposure that are essential for later discussions.
\begin{definition} [Typical and relative probability]
Let $c$ be a context, and let $M$ be a language model. We define the ``typical probability at $c$" of $M$ as $tp_c(M)=GM(p_{M}(t|c))$.\\
Let $t$ be a token, we further define the ``relative probability of $t$ at $c$ by $M$" as $rp_M(t|c) := \frac{p_M(t|c)}{tp_c(M)}$.
\end{definition}

\begin{lemma}
$tpr_c(M_1, M_2) = \frac{tp_c(M_1)}{tp_c(M_2)}$.
\label{tpr_rtp}
\end{lemma} 

\begin{proof}
$tpr_c(M_1, M_2) = \exp(\frac{1}{n}\sum \log(\frac{p_{M_1}(t|c)}{p_{M_2}(t|c)})) = \frac{\exp(\frac{1}{n}\sum \log(p_{M_1}(t|c)))}{\exp(\frac{1}{n}\sum \log(p_{M_2}(t|c)))} = \frac{tp_c(M_1)}{tp_c(M_2)}$.
\end{proof}

By this lemma, $\alpha$-exposed is equivalent to $\frac{rp_{M_1}(t|c)}{rp_{M_2}(t|c)} = \alpha$.

\begin{lemma} [Token exposure multiplicity]
If $t$ is $\alpha$-exposed by $M_1$ over $M_2$ at c and $\beta$-exposed by $M_2$ over $M_3$ at $c$, then $t$ is $\alpha \beta$-exposed by $M_1$ over $M_3$ at $c$.
\label{exposure_multi}
\end{lemma}

\begin{proof}
$\frac{rp_{M_1}(t|c)}{rp_{M_3}(t|c)} = \frac{rp_{M_1}(t|c)}{rp_{M_2}(t|c)} \cdot \frac{rp_{M_2}(t|c)}{rp_{M_3}(t|c)} = \alpha \beta$.
\end{proof}

\subsection{Aggregation}
In this subsection, we provide a formal definition of the notion of a min-bounded function and describe how DOMBA aggregates two submodels.
\begin{definition} [Proper-avg function]
Let $f: \mathbb{R^{+}}^2 \rightarrow \mathbb{R^{+}}$. We call $f$ a proper-avg function if $\forall{x,y}: \min(x,y) \leq f(x,y) \leq \max(x,y)$.
\end{definition}

\begin{definition} [Min-bounded function]
Let $f$ be a proper-avg function. we call $f$ min-bounded if $\forall{x,y}, f(x,y) \leq \lambda_f\min(x,y)$ for some constant $\lambda_f$.
\end{definition}

In practice we use the generalized mean \cite{Sykora2009} with $\alpha < 0$ for min-bounded functions, that is, $f(x,y)=(\frac{1}{2}(x^{\alpha} + y^{\alpha}))^{\frac{1}{\alpha}}, \lambda_f = 2^{-\frac{1}{\alpha}}$. Two special cases are:
\begin{enumerate}
\item $\alpha \rightarrow -\infty$ (Minimum): $f(x,y) = \min(x,y), \lambda_f=1$.
\item $\alpha = -1$ (Harmonic mean): $f(x,y) = \frac{2xy}{x+y}, \lambda_f=2$.
\end{enumerate}

Note that the arithmetic mean ($\frac{x+y}{2}$) is not min-bounded.

\begin{definition} [DOMBA aggregation]
Let $M_1, M_2$ be language models, and let f be a min-bounded function. We define $DAGG_f(M_1, M_2)$ (denoted as $M$) as a model that assigns probabilities as follows: $p_M(t|c) = \frac{M(t|c)}{\sum_{t' \in \Sigma} M(t'|c)}$, where $M(t|c) = f(rp_{M_1}(t|c), rp_{M_2}(t|c))$. 
\end{definition}

We note that DOMBA uses $f$ to average the \textbf{relative probabilities}. In contrast, averaging the probabilities would lead to inferior bounds in the subsequent subsection.

\subsection{Bounding the Exposure of DOMBA}
In this subsection, we establish the bounds on DOMBA's exposure over both submodels (Theorem \ref{bound_theorem}) as well as over any other model (Corollary \ref{corollary1}).
We begin by introducing several definitions and lemmas that will be used for proving the main theorem later on.

\begin{definition}
Let $M_1, M_2$ be language models, and let f be a min-bounded function. Let $M = DAGG_f(M_1, M_2)$. We define $\overline{f_c}(M_1, M_2) = GM(M(t|c)^{-1})$.
\end{definition}

While it might be unclear how to interpret $\overline{f_c}(M_1, M_2)$, it relates to a notion of ``mean exposure" between $M_1$ and $M_2$: 

\begin{definition} [Mean absolute exposure]
Let $c$ be a context and $M_1, M_2$ be language models. We define the ``mean absolute exposure between $M_1$ and $M_2$ at $c$" as $MAE_c(M_1, M_2) = GM(\max(\frac{rp_{M_1}(t|c)}{rp_{M_2}(t|c)}, \frac{rp_{M_2}(t|c)}{rp_{M_1}(t|c)}))$.
\end{definition}

\begin{lemma}
$\overline{f_c}(M_1, M_2) \leq \sqrt{MAE_c(M_1, M_2)}$.
\label{gae_lemma}
\end{lemma}

\begin{proof}
Let $x := \sum \log(\min(rp_{M_1}(t|c), rp_{M_2}(t|c)))$, $y:=\sum \log(\max(rp_{M_1}(t|c), rp_{M_2}(t|c)))$. We observe that $x+y = \sum \log(rp_{M_1}(t|c)) + \sum \log(rp_{M_2}(t|c)) = 0 + 0 = 0$. by definition, $y - x = n \cdot \log(MAE_c(M_1,M_2))$, which implies, $x = -\frac{n}{2}\log(MAE_c(M_1,M_2))$, we conclude that $\overline{f_c}(M_1, M_2) \leq \exp(-\frac{x}{n}) = \sqrt{MAE_c(M_1, M_2)}$.
\end{proof}

\begin{lemma}
$rp_M(t|c) = M(t|c) \cdot \overline{f_c}(M_1, M_2)$.
\label{fc_lemma}
\end{lemma}

\begin{proof}
$rp_M(t|c) = \frac{p_M(t|c)}{tp_c(M)} = \frac{p_M(t|c)}{\exp(\frac{1}{n}\sum \log(p_{M}(t|c)))} = \frac{M(t|c)}{\exp(\frac{1}{n}\sum \log(M(t|c)))} = M(t|c) \cdot \overline{f_c}(M_1, M_2)$
\end{proof}

In the following theorem, we provide a lower bound to the minimum token exposure achievable over two models.

\begin{theorem}
Let $c$ be a context and $M, M_1, M_2$ be language models. There exists a token $t$ that is $\geq$$\sqrt{MAE_c(M_1, M_2)}$-exposed by $M$ over either $M_1$ or $M_2$.
\label{low_bound_theorem}
\end{theorem}

\begin{proof}
By the proof of lemma \ref{gae_lemma}:
$\sqrt{MAE_c(M_1, M_2)} = \exp(-\frac{x}{n})  = GM(\frac{rp_M(t|c)}{\min(rp_{M_1}(t|c), rp_{M_2}(t|c))})$. The right end side is an average over tokens. Therefore there exists a token for which the term inside is greater than or equal to the left end side, which finishes the proof.
\end{proof}

In the following theorem, we demonstrate that using DOMBA provides a bound on the exposure that is a constant multiple of the best possible bound (Theorem \ref{low_bound_theorem}). This constant is solely dependent on $f$ and can even reach a value of $1$ ($f=Minimum$).

\begin{theorem}
Let $f$ be a min-bounded function, $t$ a token and $M = DAGG_f(M_1,M_2)$. $t$ is $\leq$$\gamma$-exposed over both $M_1$ and $M_2$ for $\gamma = \lambda_f \overline{f_c}(M_1, M_2) \leq \lambda_f \sqrt{MAE_c(M_1, M_2)}$.
\label{bound_theorem}
\end{theorem}

\begin{proof}
$\frac{rp_M(t|c)}{rp_{M_1}(t|c)} = \frac{M(t|c) \cdot \overline{f_c}(M_1, M_2)}{rp_{M_1}(t|c)} 
\leq \frac{\lambda_f \min(rp_{M_1}(t|c), rp_{M_2}(t|c)) \cdot \overline{f_c}(M_1, M_2)}{rp_{M_1}(t|c)} \leq \lambda_f \overline{f_c}(M_1, M_2)$. 
\end{proof}

We note that by assuming $M_1$ and $M_2$ assign similar relative probabilities to most tokens, we can anticipate the mean absolute exposure to be low. Essentially, we achieve average case behavior for all tokens.

In the following corollary, we informally think of $M_b$ as our base model (although the corollary holds in general).

\begin{corollary}
Let $c$ be a context, and let $M_1, M_2, M_b$ be language models. Let $t$ be a token that is $\alpha$-exposed by $M_1$ over $M_b$ at $c$ and $\beta$-exposed by $M_2$ over $M_b$ at $c$. Let $M:=DAGG_f(M_1, M_2)$. Then $t$ is $\leq$$\gamma\min(\alpha, \beta)$-exposed by $M$ over $M_b$ at $c$ for $\gamma$ as in theorem \ref{bound_theorem}.
\label{corollary1}
\end{corollary}

\begin{proof}
Follows directly from lemma \ref{exposure_multi} and theorem \ref{bound_theorem}.
\end{proof}

Stating the corollary in other words, if we fix a context $c$, for any token $t$, the exposure of the aggregated model $M$ over any model $M_b$ is bounded by the minimum of the exposures of the submodels $M_1$ and $M_2$ over $M_b$, multiplied by a small value. This implies that if the exposure of a token $t$ by either submodel over $M_b$ is small (i.e., $t$ is not substantially exposed by at least one submodel), then the exposure by the aggregated model over $M_b$ cannot be too large (i.e., $t$ will not be substantially exposed by the aggregated model). Given that in \methodName, each submodel is trained on separate access levels, and assuming that access levels with shared secrets are assigned to the same partition, it is expected that each secret will not be substantially exposed by at least one of the submodels, and thus, \methodName will provide a defense against the exposure of these secrets.

\section{Evaluation}

\subsection{Datasets}
Since access-controlled datasets are not publicly available, we required datasets that mimic the access-control scenario. 
These datasets need to be divided into different topics (which serve as access levels), with many data records per topic. Additionally, to use two of our security evaluation metrics, the data records should contain phrases that we refer to as ``sensitive-mimicking phrases" – phrases unique to the topic that could be considered sensitive / secret.

\subsubsection{Movie Reviews}
The first dataset we utilized is the IMDB Spoiler reviews dataset~\cite{misra2019imdb}. 
We randomly selected 50 reviews of different movies released after 2015 and considered the movie of each review selected as an access level. 
Then, we collected all of the reviews for each of the 50 movies. We note that some reviews contain details about the movie's plot, cast members, or characters, which mimic sensitive information. We utilized the Movies Metadata dataset~\cite{Rounak2018movies} to retrieve cast members' names and used them as sensitive-mimicking phrases. 
The number of reviews totaled 22,742, with 10\% of each movie's reviews set aside for evaluation. The number of reviews per movie ranged between 160 and 751. 

\subsubsection{Recipes}
The second dataset used is the Food.com Recipes and Interactions dataset \cite{shuyang_li_2019}. We utilized class labels of the Food-101 dataset \cite{bossard14} to partition the recipes into multiple sets. Each set includes recipes with titles containing a specific class label (e.g., pizza). We selected the 10 most frequent classes as the access levels. We note that the recipes include specific details about the process of creating each dish, which can mimic, for example, sensitive detailed descriptions of product manufacturing processes. We use the ingredients of each recipe as sensitive-mimicking phrases. However, since some ingredients are common among many classes, we only consider ingredients that appear in recipes of a certain class with a frequency at least 10 times greater than the frequency in all of the recipes. The number of recipes totaled 10,829, with 10\% of each class put aside for evaluation. The number of recipes per class ranged between 408 and 2283.

\subsection{Training}
For training we used LORA \cite{hu2021lora}. LORA is a fine-tuning technique that uses a small number of trainable parameters. Training is relatively fast with LORA, and the resulting model requires minimal storage space. These qualities were crucial for our experiments, as we conducted numerous trials with limited computational resources. However, it is important to note that the theoretical analysis is not dependent on the training method, and we anticipate that the experiments can be replicated using other training techniques as well. The base model used was OpenAI-GPT \cite{radford2018improving} which has 117 million parameters and a vocabulary size of 40,478. This model's original training data is a books dataset from 2015 \cite{Zhu_2015_ICCV}. This limits the prior knowledge the model possesses regarding movies and recipes. Since recent LLMs are trained on more recent and diverse datasets, evaluating them on sensitive information from the movie and recipe datasets would be challenging, as the models are probably familiar with some of the information. We note that although recent LLMs are larger and perform better than OpenAI-GPT, many are still based on the same underlying principles. Our theoretical analysis and proposed approach generalizes to any language model based on next-token prediction and does not rely on the specifics of any particular architecture. Regarding training parameters, we conducted experiments with varying numbers of training epochs (1, 2, and 4).
The hyperparameters for LORA were set to default values and were not explored: r=64, lora\_alpha=32, lora\_dropout=0.05, optimizer=paged\_adamw\_32bit, learning rate=5e-4, and warmup\_ratio=0.03. All experiments were conducted on an NVIDIA A100-SXM4-40GB GPU.

\subsection{Compared Models}

\subsubsection{Non-secure models (NSec)} In these models, which serve as baselines, no attempt is made to secure sensitive information. \textbf{FT-ALL}: OpenAI-GPT fine-tuned on the entire training dataset.
\textbf{AGG-A}: Similar to DOMBA-INIT, but using arithmetic mean, a non-min-bounded function.

\subsubsection{Secure models (Sec)} While these models are trained on all the data with an effort made to secure sensitive information, they do not include an access-control mechanism.
\textbf{SUBMIX}: A DP aggregated model constructed using the method of \citeauthor{ginart2022submix} \shortcite{ginart2022submix}, with three submodels (two parts + the base model). For a meaningful comparison, we tuned the privacy parameter $\beta$ to 0.3, which resulted in utility comparable to DOMBA on the movies dataset.
\textbf{D-I-H}: DOMBA-INIT (without DOMBA-FT), using harmonic mean for aggregation.
\textbf{D-I-M}: DOMBA-INIT (without DOMBA-FT), using minimum for aggregation.

\subsubsection{Access-controlled models (AC)} These models are designed to secure sensitive information while providing an access-control mechanism.
\textbf{Per-AL}: A separate model for each access level, achieved by fine-tuning OpenAI-GPT only on data records of that access level.
\textbf{DOMBA}: Our full method, using the minimum function for aggregation.

\begin{table*}[t]
\centering
\begin{tabular}{|l|l||ll|ll||ll|ll|ll|ll|}
\hline
            &  & \multicolumn{4}{|c||}{Utility metrics} & \multicolumn{8}{|c|}{Security metrics} \\
\hline
           Type & Model & \multicolumn{2}{|l|}{HOPPL $\downarrow$} & \multicolumn{2}{|l||}{AUPPL $\downarrow$} & \multicolumn{2}{|l|}{EXP $\downarrow$} & \multicolumn{2}{|l|}{SPPL $\uparrow$}  & \multicolumn{2}{|l|}{SIA $\downarrow$}  & \multicolumn{2}{|l|}{CAN $\downarrow$}  \\
           & & R     & M     & R     & M      & R     & M        & R      & M          & R    & M     & R     & M     \\
\hline
NSec  &   FT-ALL    & 19.9  & 48.39 & 15.61 & 41.55  & -      & -          & 15.31 & 61.81    & 0.81 & 0.83  & 14.47 & 28.03 \\
&  AGG-A      & 22.34 & 49.37 & 17.83 & 43.55  &  207.9 & 1699 & 19.78 & 78.93  & 0.82 & 0.82  & 14.82 & 23.88 \\
\hline
Sec  & SUBMIX   & 29.58 & \underline{50.84} & 25.4  & 48.54  &  5.1  & 17.36   & 49.95 & 550   & 0.76 & 0.8   & 4.73  & 4.01  \\
&  D-I-H    & \textbf{23.57} & \textbf{50.83} & \underline{20.09} & \underline{48.37}  &  2.38    & 2.84  &  50.87 & 895.1   & 0.64 & 0.66  & 3.49  & \underline{2.43}  \\
&  D-I-M    & \underline{24.54} & 51.89 & 21.14 & 49.79 & \textbf{1.77} & \textbf{2.31} & \underline{61.99} & \textbf{1161}  &  \underline{0.6}  & \underline{0.64}  & \underline{3.21}  & \textbf{2.21}  \\
\hline
AC  & PER-AL    & 45.15 & 63.39 & 27.87 & 54.18  & -     & -        & -      & -          & -    & -     & -     & -    \\
&  DOMBA & 25.19 & 52.22 & \textbf{16.85} & \textbf{42.48} & \underline{1.78}   & \underline{2.37}   & \textbf{74}    & \underline{1127}     & \textbf{0.54} & \textbf{0.62}  & \textbf{2.88}     & 2.44     \\
\hline

\end{tabular}
\caption{Results with two epochs of training. For each metric and model, the results for both the recipe dataset (R) and movie reviews dataset (M) are presented. The best values for secure and access-controlled models are in bold, and the second best values are underlined. (Note that PER-AL is trivially secure and therefore, no results are presented for its security; the EXP metric is only meaningful for methods which aggregate two submodels and therefore, it is not presented for FT-ALL.)}
\label{res2ep}
\end{table*}

\subsection{Metrics}
In this section, we describe the metrics used to evaluate the models' utility and security. For utility we use perplexity, which measures the model's ability to predict the next token in a text. For security we use four different metrics: exposure, secret perplexity, a secret inference attack AUC-ROC, and the canary technique score \cite{236216}. We note that for access-controlled models, we evaluate the security of each variant (corresponding to an access level) using data with a different access level than the one that the variant was trained for.

\subsubsection{Utility Evaluation}
We evaluate utility in terms of perplexity on two evaluation sets as follows: 1. \textbf{HOPPL}: perplexity on held out data with access levels that were not used for training. This metric provides a ``fair" way of comparing secure and non-secure models, as the non-secure models are not expected to gain by ``knowing" restricted information. 2. \textbf{AUPPL}: perplexity on held out data of the access levels used for training (for access-controlled models - the corresponding variant is used for each access level).  The main purpose of this metric is to compare the utility of secure and access-controlled models. We expect the access-controlled models to gain utility by ``knowing" authorized restricted information. For both metrics above, we calculate the perplexity as: $perp_M(D_e) = \exp(\frac{1}{|D_e|}\sum_{r \in D_e}\sum_i{-\log(p_M(r_i|r_{<i}))}$, where $|D_e|$ is the amount of tokens in $D_e$, $r$ is a data record, $r_i$ is the i'th token in the record, $r_{<i}$ are the tokens preceding it, and $p_M$ is the probability assigned by the model.

\subsubsection{Exposure (EXP)}
In our theoretical analysis (Theorem \ref{bound_theorem}), we established that the exposure of $M=DAGG_f(M_1,M_2)$ over both $M_1$ and $M_2$ is bounded for any token by $\lambda_f \overline{f_c}(M_1, M_2) \leq \lambda_f \sqrt{MAE_c(M_1, M_2)}$. To validate this, we measure ``extreme case" exposure of $M$ over $M_1$ and $M_2$. We report the maximum and 99th percentile exposure ($ = \frac{rp_M(t|c)}{\min(rp_{M_1}(t|c), rp_{M_2}(t|c))}$) for all tokens observed in the data, given the previous tokens as context.

\subsubsection{Secret Perplexity (SPPL)}
One way of measuring the model's ability to handle sensitive information is by evaluating perplexity specifically on sensitive-mimicking phrases. Given a model M, we measure the perplexity of each instance of a sensitive-mimicking phrase in the evaluation dataset. Specifically, let $x:=x_1,...,x_k$ be the token representations of a sensitive-mimicking phrase and c be the tokens preceding this phrase, we measure $perp_M(x|c) = \exp(\frac{1}{k}\sum_i{-\log(p_M(x_i|c,x_1,...,x_{i-1}))})$. We report the average of the mean perplexity of each access level.
This metric aims to provide a basic, rough evaluation of a model's ability to handle sensitive information.

\subsubsection{Secret Inference Attack (SIA)}
This attack is based on a membership inference attack with a reference model \cite{mireshghallah-etal-2022-quantifying, pmlr-v130-kumar-murakonda21a}. The original attack works as follows: Given a reference model $M_b$, a target model $M$, and a potential training data record $r$ of $M$, measure the log ratio of the probabilities of $r$ according to $M$ and $M_b$, that is $\log(\frac{p_M(r)}{p_{M_b}(r)})$. If this value is above a certain threshold, consider $r$ as belonging to the training data of $M$. In our scenario, instead of inferring the membership of any data record, the adversary tries to infer secrets. Therefore, we only consider probabilities assigned to sensitive-mimicking phrases: cast members' names for the movie reviews dataset and secret ingredients for the recipe dataset. The attack dataset consists of tuples $(c, t, label)$, where $c$ is a context, $t$ is a phrase, and $label$ is $true$ if $t$ is sensitive and $false$ otherwise. To obtain data points labeled $false$, we replace each sensitive-mimicking phrase $t$ by $t'$, which is another phrase of the same type (cast member name or ingredient) that is not a sensitive-mimicking phrase. For every data point $(c, t, true)$, we have a data point $(c, t', false)$. We report the AUC-ROC of the attack.

\subsubsection{The Canary Technique (CAN)}
We adapt the attack proposed by \citeauthor{236216} \shortcite{236216} to the access-control scenario. For each access level, we insert 30 repetitions of a phrase (canary) consisting of seven randomly chosen words into the training set for that access level (the number of repetitions and phrase length were selected arbitrarily). This canary mimics sensitive information for the access level. We report the median attack score across access levels. An attack score of $s$ means that only $(\frac{1}{2})^{s}$ of phrases of the same length have a higher probability of being generated by the model. A score near one suggests that the model did not memorize the canary.

\section{Results}
The results with two epochs of training are shown in Table \ref{res2ep}. As expected, FT-ALL achieved the best utility across both metrics. Among secure and access-controlled models, D-I-H had the highest HOPPL utility, while DOMBA excelled on the AUPPL metric, demonstrating the value of the DOMBA-FT step. Comparing access-controlled models, DOMBA substantially outperformed PER-AL across both datasets and both utility metrics.

Regarding security, non-secure models performed substantially worse, compared to secure models, on all metrics. Among the secure and access-controlled models, SUBMIX obtained the worst values for all metrics and datasets, D-I-M and DOMBA obtained the best values, and D-I-H was slightly worse. Although secure models provide substantially better security compared to non-secure models, they are not perfect. For instance, a perfectly secure model would score 0.5 on SIA and one on CAN. This does not imply that secure models are not actually secure. For example, the values obtained by all secure and access-controlled models for the canary technique metric are considered impractical for extracting useful information \cite{236216}. 

Regarding the impact of the choice of min-bounded function, D-I-M achieved better security and worse utility compared to D-I-H, as expected. We also tested replacing the min-bounded function used in DOMBA from minimum to harmonic mean and the effect was similar.

\begin{figure}[t]
\centering
\includegraphics[width=0.46\textwidth]{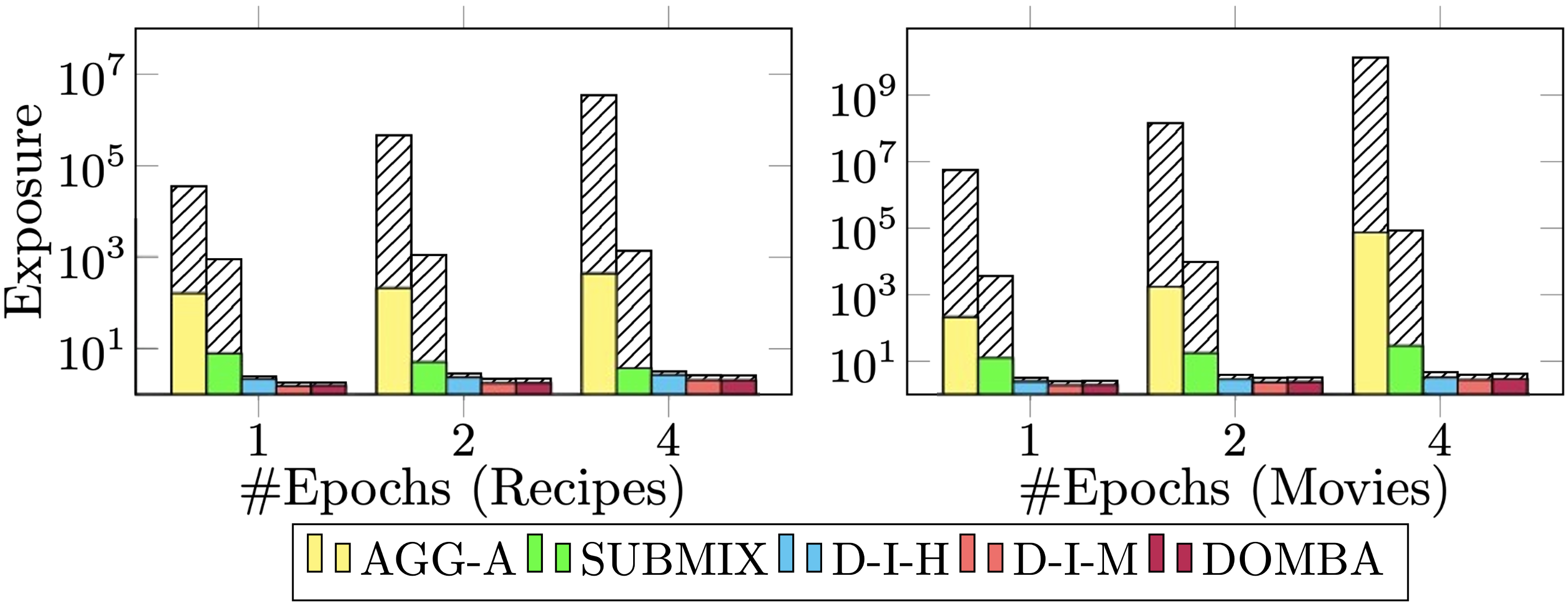}
\caption{Exposure (log scale, lower is better) of different models with 1, 2, or 4 training epochs. The colored bars represent the 99th percentile exposure, while the dashed bars represent the maximal observed exposure.}
\label{all_exp}
\end{figure}

Figure \ref{all_exp} shows the worst-case and 99th percentile exposure of models employing different aggregation methods on the recipe and movie reviews datasets for 1, 2, and 4 training epochs. The maximum exposure of DOMBA, D-I-H, and D-I-M is 4.69. In comparison, SUBMIX reaches a maximum exposure of 8.5e4 and AGG-A reaches a maximum exposure of 1.3e10. We observe that \methodName's 99th percentile exposure is similar to its maximum exposure, supporting the theoretical bound established by our analysis (Theorem \ref{bound_theorem}). Regarding the effect of the number of epochs, increasing it generally leads to higher exposure. However, the increase in exposure for DOMBA, D-I-H, and D-I-M is moderate compared to AGG-A for both datasets, while for SUMBIX, the change in exposure is inconsistent between the two datasets.

Figure \ref{tradeoff} illustrates the trade-off between utility and security for different methods across both datasets. For most models, as the number of training epochs increases, security tends to worsen while utility improves. However, non-secure models experience a much greater decline in security. DOMBA achieves the best trade-off, providing superior security while maintaining utility levels similar to those of the non-secure models.

\begin{figure}[t]
\centering
\includegraphics[width=0.46\textwidth]{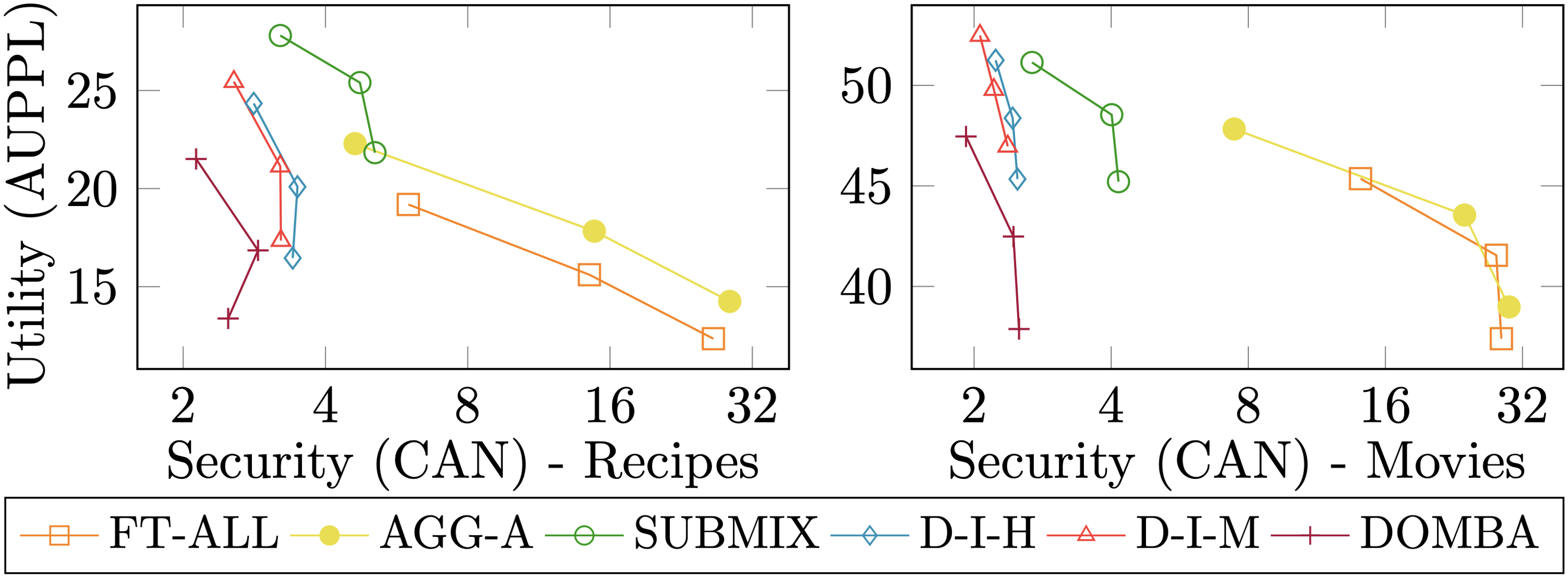}
\caption{Utility-security trade-off (for both metrics lower is better, with the security metric on a log scale). Each line represents a model trained for different numbers of epochs (1, 2, and 4). The point representing 1 epoch is always the leftmost and highest point on the line.}
\label{tradeoff}
\end{figure}

\section{Discussion}

To analyze DOMBA, we developed the token exposure framework, offering an alternative to DP for theoretically analyzing the security of language models. The key idea behind token exposure is to evaluate the difference between a safe model and a potentially unsafe one. In contrast, DP evaluates the differences between models trained on nearly identical datasets (in both approaches, a smaller difference indicates stronger security). Additionally, token exposure utilizes relative probabilities over absolute ones. We find the token exposure approach more intuitive and practical for analyzing DOMBA. Future work could further explore the token exposure framework and investigate its relationships and potential synergies with DP.


\methodName relies on a strict separation of access levels into two distinct partitions that do not share sensitive information. Achieving this separation requires the graph of shared secrets between access levels to be disconnected. We hypothesize that this condition is usually met, as organizations often handle documents across diverse topics. Future research could validate this hypothesis and investigate approaches for scenarios where the graph remains connected.

While the training time of DOMBA scales linearly with the dataset size, similar to training a single LLM, and storage requirements can be kept low by using LORA, inference incurs additional resource overhead due to the deployment of two LLMs instead of one. This overhead may render the approach impractical for certain applications. One potential solution is to employ \methodName as a teacher model to train a student model via knowledge distillation \cite{xu2024survey, PATE}, where the student model serves as a deployed model mimicking DOMBA. 

\section{Conclusion}
In this paper we proposed \methodName, a novel approach for training and deploying access-controlled LLMs with high utility. We formalized the concept of exposed secrets by developing the token exposure framework and bounded \methodName's exposure. We evaluated \methodName's performance on two access-controlled datasets, mimicking real world organizations' needs. Our evaluation showed that \methodName achieves a better security-utility trade-off than existing methods, across both datasets, two utility metrics and four security metrics. Finally, we believe that the principles of min-bounded aggregation and relative probabilities, which serve as \methodName's core, have substantial potential to serve as foundational elements in a wide range of future machine learning research, extending beyond the scope of security.

\section*{Acknowledgments}
The authors would like to thank Dr. Avi Segal for their valuable insights and thoughtful feedback, which contributed to the refinement of this manuscript.

\bibliography{aaai25}

\end{document}